\documentclass[11pt]{article}
\usepackage{varioref}
\usepackage{amsmath}
\usepackage{amsthm}
\usepackage{algorithmic}
\usepackage[boxed]{algorithm}
\usepackage[margin=1in]{geometry}

\DeclareMathOperator{\ds}{\mathit{dnf-size}}

\newcommand{\ignore}[1]{}
\newenvironment{proofof}[1]{\noindent \textbf{Proof of {#1}:}}{\qed \medskip}
\newenvironment{sketchof}[1]{\noindent \textbf{Sketch of Proof of {#1}:}}{\qed \medskip}

\newcommand{\Maj}{\mathsf{Maj}}

\newcommand{\poly}{\mathrm{poly}}

\begin{document}
\title{Tight Bounds on Proper Equivalence Query Learning of DNF}
\author{Lisa Hellerstein\thanks{Polytechnic Institute of NYU. {\tt hstein@poly.edu}. Supported by by NSF Grant CCF-0917153.} \and Devorah Kletenik\thanks{Polytechnic Insitute of NYU. {\tt dkletenik@cis.poly.edu}. Supported by the US Department of Education GAANN grant P200A090157.}\\ \and Linda Sellie\thanks{Polytechnic Insitute of NYU. {\tt sellie@mac.com}. Supported by NSF grant 0937060 to the CRA for the CIFellows Project.}\\ \and Rocco A. Servedio\thanks{Columbia University. {\tt rocco@cs.columbia.edu}. Supported by NSF grants CNS-0716245, CCF-0915929, and CCF-1115703.}}

\maketitle

\theoremstyle{plain}
\newtheorem{theorem}{Theorem}
\newtheorem{lemma}[theorem]{Lemma}
\newtheorem{proposition}[theorem]{Proposition}
\newtheorem{corollary}[theorem]{Corollary}

\theoremstyle{definition}
\newtheorem{property}{Property}
\newtheorem{definition}[theorem]{Definition}

\newtheorem{fact}[theorem]{Fact}

\newtheorem{claim}[theorem]{Claim}
\newtheorem{observation}[theorem]{Observation}

\begin{abstract}
We prove a new structural lemma for partial 
Boolean functions $f$, which we call the {\em seed lemma for DNF}.
Using the lemma, we give the
first subexponential algorithm for proper learning
of DNF in Angluin's Equivalence Query (EQ) model.  
The algorithm has
time and query complexity $2^{(\tilde{O}{\sqrt{n}})}$, which is optimal.
We also give a new result on certificates for DNF-size,
a simple algorithm for properly PAC-learning DNF,
and new
results on EQ-learning $\log n$-term DNF and decision trees. 
\end{abstract}

\newpage
\section{Introduction}

Over twenty years ago, Angluin began study of
the equivalence query (EQ) learning model~\cite{Angluin:88,Angluin:90}.
Valiant~\cite{Valiant:84} had asked whether DNF formulas
were poly-time learnable in the PAC model; this question is still open.
Angluin asked the same question in the EQ model.
Using {\em approximate fingerprints}, she proved that
any {\em proper} algorithm for EQ-learning DNF formulas
requires super-polynomial query complexity, and hence 
super-polynomial time.  In a proper DNF learning algorithm,
all hypotheses are DNF formulas.

Angluin's work left open the problem of determining the exact complexity
of EQ-learning DNF, both properly and improperly.
Tarui and Tsukiji noted that Angluin's fingerprint proof can be
modified to show that a proper EQ algorithm must have query complexity at least
$2^{(\tilde{O}{\sqrt{n}})}$~\cite{TaruiTsukiji:99}.  (They did not
give details, but we prove this explicitly as a consequence of
a more general result.)
The most efficient {\em improper} algorithm for EQ-learning 
DNF is due to Klivans and Servedio (Corollary 12 of \cite{KlivansServedio:04jcss}), and runs in time
$2^{\tilde{O}({n^{1/3}})}$. 

In this paper, we 
give the first subexponential algorithm
for {\em proper} learning of DNF in the
EQ model.   Our algorithm has
time and query complexity that, like
the lower bound, is $2^{(\tilde{O}{\sqrt{n}})}$.

Our EQ algorithm implies a new result on certificates
for DNF size.  
Hellerstein et al. asked whether DNF has
``poly-size certificates''~\cite{HPR+:96},
that is, whether there are polynomials $q$ and $r$ such that
for all $s,n > 0$, functions requiring DNF formulas of size
greater than $q(s,n)$
have certificates of size
$r(s,n)$ certifiying that they do not have DNF formulas of
size at most $s$.
(This is equivalent to asking
whether DNF can be properly MEQ-learned within
polynomial query complexity~\cite{HPR+:96}.)
Our result does not resolve this question, but it shows
that there are analogous subexponential certificates. 
More specifically, it shows that
there exists a function $r(s,n) = 2^{O(\sqrt{n\log s}\log n)}$ such that
for all $s,n > 0$, functions requiring DNF formulas
of size greater than $r(s,n)$ have certificates
of size $r(s,n)$ certifying that they do not have
DNF formulas of size at most $s$. 

Our EQ algorithm is based on a new structural
lemma for partial Boolean functions $f$, which
we call the {\em seed lemma for DNF}.  
It states that if
$f$ has at least one positive example and is consistent
with a DNF of size $s$,
then $f$ has a projection $f_p$, induced by
fixing the values of $O(\sqrt{n\log s})$ variables, such that
$f_p$ has at least one positive example, and is consistent with
a monomial.

We also use the seed lemma for DNF to 
obtain a new subexponential proper algorithm for PAC-learning DNFs which 
is simpler than the previous algorithm of 
Alekhnovich et al.~\cite{ABFKP:08}, with the same bounds.
That algorithm uses a procedure that runs multiple
recursive calls in round robin fashion until one succeeds.
In contrast, ours is an iterative procedure with a straightforward
analysis.  

Decision-trees can be PAC and EQ-learned in time $n^{O(\log s)}$,
where $s$ is the size of the tree~\cite{EhrenfeuchtHaussler:89, Simon:95}.
We prove a seed lemma for decision trees as well, and use it 
to obtain an algorithm that learns
decision trees using DNF hypotheses in time
$n^{O(\log s_1)}$, where
$s_1$ is the number of 1-leaves in the tree. 
(For any ``minimal'' tree, the number of 0-leaves
is at most $n s_1$; this bound is tight for the optimal
tree computing a monomial of $n$ variables.)

We prove a lower bound result that
quantifies the tradeoff between the number of queries 
needed to properly EQ-learn DNF formulas,
and the size of such queries.
One consequence is a lower bound of
$2^{\Omega(\sqrt{n \log n})}$ on the query
complexity necessary for an EQ algorithm 
to learn DNF formulas of size poly($n$), using
DNF hypotheses.  
This matches
the lower bound of $2^{(\tilde{O}{\sqrt{n}})}$ mentioned by
Tarui and Tsukuji.  The bound for our EQ algorithm, applied to 
DNF formulas of size $\poly(n)$, 
differs from this lower bound by only a factor of $\log n$
in the exponent.


We also prove a result on learning $\log n$-term DNF
using DNF hypotheses.
Several poly-time algorithms are known for this problem in the membership
and equivlence query (MEQ) 
model~\cite{Bshouty:97,BlumRudich:95,BGHM96,HellersteinRaghavan:05}.
We prove that the membership queries are essential:
there is no $\poly(n)$-time algorithm that learns
$O(\log n)$-term DNF using DNF hypotheses, with equivalence
queries alone.
In contrast, Angluin and Kharitonov showed that, under cryptographic
assumptions, membership queries do not help in
PAC-learning unrestricted DNF formulas~\cite{angkha95}.
Blum and Singh gave an algorithm that PAC-learns
$\log n$-term DNF using DNF hypotheses of size
$n^{O(\log n)}$ in time $n^{O(\log n)}$ \cite{BlumSingh:90}; our results 
imply that no significant improvement of this result is possible
for PAC-learning $\log n$-term DNF using DNF hypotheses.

\section{Preliminaries}

Assignment $x \in \{0,1\}^n$  is a {\em positive example} of
Boolean function $f(x_1, \ldots, x_n)$ if $f(x) = 1$, and a {\em negative example} if $f(x) = 0$.
A {\em sample} of $f$ is a set of pairs $(x,f(x))$, where $x \in \{0,1\}^n$.

A literal is a variable or its negation.
A \emph{term}, also called a {\em monomial}, is a possibly empty conjunction ($\wedge$) of literals.
If the term is empty, all assignments satisfy it.
The {\em size of a term} is the number of literals in it.
We say that term $t$ {\em covers} assignment $x$ if $t(x) = 1$. It is an implicant of Boolean function $f(x_1, \ldots, x_n)$
if $t(x) = 1$ implies $f(x) = 1$.
A \emph{DNF} (disjunctive normal form) formula is either the
constant 0, the constant 1, or 
a formula of the form
$t_1 \vee \dots \vee t_k$, where $k \geq 1$ and each $t_i$ is a term.
A $k$-term DNF is a DNF formula consisting of at most $k$ terms.
A $k$-DNF is a DNF formula where each term has size at most $k$. 
The {\em size} of a DNF formula is the number of its terms. 

A \emph{partial} Boolean function $f$ maps $\{0,1\}^n$ 
to  $\{0,1, \ast\}$, where $\ast$ means undefined.
A Boolean formula $\phi$ is 
\emph{consistent} with a partial function $f$ (and vice versa)
if $\phi(x) = f(x)$ for all $x \in \{0,1\}^n$ where $f(x) \neq \ast$. 
If $f$ is a partial function, then  $\ds(f)$ is the
size of the smallest DNF formula consistent with $f$. 

Let $X_n = \{x_1, \ldots, x_n\}$.
A {\em projection} of a (partial) function $f(x_1,\ldots,x_n)$ is a function
induced from $f$ by fixing $k$ variables of $f$ to constants in $\{0,1\}$, where
$0 \leq k \leq n$.  We consider the domain of the projection to be the set of assignments to the remaining $n-k$ variables.
If $T$ is a subset of literals over $X_n$, or a term over $X_n$,
then $f_T$ denotes the projection of $f$
induced by setting the literals in $T$ to 1.

For $x \in \{0,1\}^n$ we write $|x|$ to denote
$\sum_i x_i$ and  
$\Maj(x_1,\dots,x_n)$ to denote the majority function
whose value is 1 if $\sum_{i=1}^n x_i \geq n/2$
and 0 otherwise.
We write
``$\log$'' to denote log base 2.

A \emph{certificate} that a property $P$ holds for a 
Boolean function $f(x_1, \dots, x_n)$ is a set $A \subseteq \{0,1\}^n$
such that for all Boolean functions
$g(x_1, \dots, x_n)$, if $g$ does not have property $P$, then 
$f(a) \neq g(a)$ for some $a \in A$. 
The \emph{size} of certificate $A$ is the number of assignments in it.

We use standard models and definitions
from computational learning theory.  
We omit these here; more information
can be found in Appendix~\ref{models}.

We sometimes use the notation $\tilde{O}()$, rather than $O()$, to
denote that we are suppressing factors that
are logarithmic in the arguments to $\tilde{O}()$.

\section{Seeds}

We introduce the following definition.

\begin{definition}
\label{seedDef1}
A {\em seed} 
of a partial Boolean function $f(x_1, \ldots, x_n)$ 
is a (possibly empty) monomial $T$ that
covers at least one positive example of $f$, such that $f_T$ 
is consistent with a monomial.

\end{definition}


Our new structural lemma is as follows.

\begin{lemma}
\label{seedTh}
{\bf (Seed lemma for DNF)}
Let $f$ be a partial Boolean function such that $f(a)=1$ for
some $a \in \{0,1\}^n$.  Let $s = \ds(f)$.
Then $f$ has a seed of size
at most $2\sqrt{n \ln{s}}$.
\end{lemma}

\begin{proof}

Let $\phi$ be a DNF formula of size $s=\ds(f)$ that is consistent with $f$.
If $\phi = 1$, then $\emptyset$ is a seed.   Suppose $\phi \neq 1$. Then since $f(a) = 1$, $\phi$ has at least one term.
Since $\phi$ has size $s=\ds(f)$, it is of minimum size,
each term of $\phi$ covers at least one positive
example of $f$.
We construct seed $T$ from $\phi$ by
initializing two sets $Q$ and $R$ to be empty, and then 
repeating the following steps until
a seed is output: 

\begin{enumerate}
\item If there is a term $P$ of $\phi$
of size at most $\sqrt{n  \ln s}$, output the conjunction of the literals
in $Q \bigcup P'$ as a seed, where $P'$ is
the set of literals in $P$.
\item 
If all terms of 
$\phi$ have size greater than
$\sqrt{n  \ln s}$, check whether there is a literal $l \not\in Q \cup R$
that is satisfied by all positive examples of $f_Q$.
\begin{enumerate}
\item If so, add $l$ to $R$.  Set $l$ to 1 in $\phi$ by
removing
all occurences of $l$ in the terms of $\phi$. (There are no occurences
of $\bar{l}$ in $\phi$.)
\item If not, let
$l$ be the literal appearing in the largest
number of terms of $\phi$.
Add $\bar{l}$ to $Q$.
Set $l$ to 0 in $\phi$ by 
removing from $\phi$
all terms containing $l$, and removing all occurences of $\bar{l}$
in the remaining terms. 
Also remove any
terms which no longer cover a positive example of $f_{Q \cup R}$.
\end{enumerate}
\end{enumerate}

We now prove that
the above procedure outputs a seed
satisfying the properties of the lemma.
During execution of Step 2a, 
no terms are deleted.
At the start of execution of Step 2b, 
there is a positive example of $f_{Q \bigcup R}$
that does not satisfy $l$, and hence a term $t$ of
$\phi$ that does not contain $l$; the
updates made to $\phi$ in Step 2b
do not delete $t$.
Thus the following three invariants are maintained by
the procedure:
(1) $\phi$ contains at least one (possibly empty) term, and each term of $\phi$ covers
at least one positive example of $f_{Q \bigcup R}$
(2)
$\phi$ is consistent with $f_{Q \bigcup R}$ and (3)
each term of $\phi$ covers at least one positive
example of $f_{Q \bigcup R}$.

Literals are only added to $R$ in Step 2a, when there is
a literal $l$ satisfied by all positive examples of $f_Q$.
Thus another invariant holds:
(4) for any positive example $a$ of $f$, if $a$ satisfies all literals in $Q$, 
then $a$ satisfies all literals in $R$.

Since each loop iteration 
removes a variable
from $\phi$, there are at most $n$ iterations.
By the invariants, when $T$ is output,
$\phi$ is consistent with $f_{Q \bigcup R}$, and term $P$
of $\phi$ is satisfied by at least one positive
example of $f_{Q \bigcup R}$.
Thus $f_{Q \bigcup P'}$  has at least one positive example.
Further, since $P$ is a term of $\phi$, and $\phi$ is consistent
with $f_{Q \bigcup R}$, if an assignment $a$ satisfies
$Q \bigcup P' \bigcup R$ then $f(a) =1$ or $f(a) = \ast$.
Thus $f_{Q \bigcup P'}$ is consistent with the monomial $\bigwedge_{l \in R} l$,
and $Q \bigcup P'$ is a seed.

Clearly $P$ has at most $\sqrt{n \ln s}$ literals.
We use a standard technique to bound
the size of $Q$ (cf. ~\cite{Angluin:90}). 
Each time a literal is added to $Q$, 
all terms of $\phi$ have size at least $\sqrt{n \ln s}$, and thus
the literal appearing in the most terms of $\phi$
appears in 
at least $\alpha s$ terms,  for $\alpha = \sqrt{ (\ln s)/ n}$. 
So each time a literal is added to $Q$, at least 
$\alpha s$  terms are removed from $\phi$.  
When $Q$ contains $r$ literals,
$\phi$ contains at most
${(1-\alpha)}^rs$ terms.
For $r \geq \sqrt{n \ln s}$,
${(1-\alpha)}^rs < e^{-\alpha rs} s = 1$.
Since $\phi$ always contains at least one term, 
$Q$ contains at most  $\sqrt{n \ln s}$ literals.
Thus $T$ has size at most
$2 \sqrt{n \ln s}$. 
\end{proof}

The above bound on seed size is nearly tight for a
monotone DNF formula on $n$ variables having $\sqrt{n}$ disjoint
terms, each of size $\sqrt{n}$.  The smallest
seed for the function it represents has size $\sqrt{n} - 1$.

\section{PAC-learning DNF (and decision trees) using seeds}
\label{PACsection}

We begin by presenting our algorithm for PAC-learning DNFs. 
It is simpler than our EQ algorithm, and
the ideas used here are helpful in understanding that algorithm.
We present only the portion of the PAC algorithm that 
constructs the hypothesis from an input sample $S$,
and we assume that the size $s$ of the target DNF formula is known.
The rest of the algorithm description is routine (see e.g.~\cite{ABFKP:08}). 
Let $S^+$ and $S^-$ denote the positive and negative examples in $S$, and let $f^S$ denote
the partial Boolean function that is defined consistently 
with all assignments in 
$S$, and is undefined on all assignments not in $S$.
We describe the algorithm here and give the pseudocode in Appendix~\ref{pacalgorithm}.

The algorithm begins with a hypothesis DNF $h$ that is
initialized to 0.  It finds terms one by one and adds them to $h$.
Each additional
term covers at least one uncovered example in $S^+$, and
terms are added to $h$ until all examples in $S^+$ are covered.

The procedure for finding a term is as follows.
First, the algorithm tests each conjunctions $T$ of size
at most $2^{\sqrt{n \ln s}}$ to determine whether it is a seed of $f^S$.
To perform this test, the algorithm 
explicitly checks whether $T$
covers at least one positive example in $S$; if not, $T$ is not a seed.
It then checks 
whether $f^S_{T}$ is consistent with a monomial, 
using the same approach as the standard PAC algorithm for learning
monomials~\cite{Valiant:84}, as follows.  Let $S_T$ be the
set of positive examples in $S$ that satisfy $T$.
The algorithm computes term $T'$, which is
the conjunction of
the literals that are satisfied by all examples in $S_T$
(so $T'$ includes $T$).  It is easy to show that
$f^S_T$ is consistent with a monomial iff 
all negative examples of $S$ falsify $T'$.
So, the algorithm
checks whether all negative examples in $S$ falsify $T'$.
If so, $T$ is a seed, else it is not.

By the seed lemma for DNF, at least one seed $T$ will be found.
For each seed $T$ found,
the associated term $T'$ is added to $h$,
and the positive examples satisfying $T'$ are removed from $S$.
If $S$ still contains a positive example, the procedure is 
repeated with the new $S$.

The correctness of the algorithm follows immediately
from the above discussion.
Once a seed $T$ is found, all positive examples in $S$ that satisfy
$T$ are removed $S$, and thus the same seed will never be
found twice.
Thus the algorithm runs in time
$2^{O(\sqrt{n\log{s}}\log n)}$ and outputs a DNF formula
of that size.

We can generalize the technique used in the above algorithm.
Say that an algorithm uses the {\em seed covering method}
if it builds a hypothesis DNF from an input sample $S$
by repeatedly executing the following steps, until
no positive examples remain in the sample:
(1) find a seed $T$ of partial function $f^S$, 
(2) form a term $T'$ from the positive examples in $S$ that satisfy $T$, by taking
the conjunction of the literals satisfied by all those examples,
(3) add term $T'$ to the hypothesis DNF and remove from $S$ all positive examples 
covered by $T'$.

In fact, the algorithm of Blum and Singh, which PAC-learns $k$-term DNF,
implicitly uses the seed covering method. 
It first finds seeds of size $k-1$,
then seeds of size $k-2$, and so forth.
It differs from our DNF-learning algorithm in that it only searches for
a restricted type of seed.
Our seeds are constructed from two types of literals,
those (in $Q$) that eliminate terms from the target, and 
those (in $P$) that satisfy a term. 
Their algorithm only searches for seeds containing the first type of
literal.  Algorithmically, their algorithm works by identifying
subsets of examples satisfying the
same subset of terms of the target, while ours 
works by identifying subsets of examples
satisfying a common term of the target.

We conclude this section by observing that the seed method can
also be used to learn decision trees in time $n^{O(\log s_1)}$,
where $s_1$ is the number of 1-leaves in the decision tree.
This follows easily from the following lemma.\footnote{We note that an alternative approach to proving the seed lemma for DNF
is to use Bshouty's result that states that every DNF of size $s$ has a decision
tree of size $2^{\tilde{O}(\sqrt{n})}$ with $\tilde{O}(\sqrt{n})$-DNF
formulas in the leaves~\cite{bsh96}, and then to modify our proof of the
seed lemma for decision trees to accomodate DNFs in the leaves.}

\begin{lemma} (Seed lemma for lecision trees)
Let $f$ be a partial Boolean function, such that $f$ has at least
one positive example, and $f$ is consistent with a decision tree 
having $s_1$ leaves that are labeled 1.  Then $f$ has a seed of
size at most $\log s_1$.
\end{lemma}

\begin{proof}
Let $J$ be a decision tree consistent with $f$, and let $s_1$ be
the number of its leaves that are labeled 1.
Without loss of generality, assume
that each 1-leaf of $J$ is reached by at least one positive example
of $f$.
Define an internal node of $J$ to be a {\em key}
node if neither of its children is a leaf labeled 0.
Define the {\em key-depth} of a leaf
to be the number of key nodes on the path from the root down to it.
It is not hard to show that
since $J$ has $s_1$ leaves labeled 1, 
it must have a 1-leaf with
key-depth at most $\log s_1$.
Let $p$ be the path from the root to this 1-leaf.
Let $L$ be the set of literals that are satisfied along path $p$.
Let $Q$ be the conjunction of literals in $L$ that come
from key nodes, and let $R$ be the conjunction of the
remaining literals.  Consider an example $x$ that satisfies $Q$.
Consider its path in $J$.  If $x$ also satisfies $R$, it will
end in the 1-leaf at the end of $p$, else it will
diverge from $p$ at a non-key node, ending at
at the 0-child of that node.  Thus $f_Q$ is consistent
with monomial $R$, $Q$ is a seed of $f$, and $|Q| \leq \log n$.
\end{proof}

\section{EQ-learning DNF using seeds}
\label{sec:eq-alg}
We now present our algorithm for EQ-learning DNF\@.  It can be viewed as learning a decision
list with monomials of bounded size in the nodes, and (implicant) monomials of unbounded size
in the leaves (and a 0 default); we use a variant of the approach used
to EQ-learn decision lists with bounded-size monomials 
in the nodes, and constant leaves~\cite{HSW:90, Simon:95}.
Like our PAC algorithm, our EQ algorithm could be generalized to learn other classes with seeds.

Let $\phi$ be the target DNF, and  let $s$ be the size of $\phi$.
Let $f$ be the function represented by $\phi$.
Let $X = \{x_1, \ldots, x_n\}, \bar{X} = \{\bar{x}_1, \ldots, \bar{x}_n\}$.
Let $Q = \{t \subseteq X \cup \bar{X} \mid |t| \leq 2\sqrt{n \ln s } \}.$ $Q$ is the set of potential seeds. 

We  first introduce the main ideas of the algorithm.
Define a sequence of partial functions as follows.
Let $f^{(1)} = f$.  For $1< i \leq |Q|$, let $f^{(i)}$ 
be the partial function that
is identical to $f^{(i-1)}$ except on positive assignments $a$ 
of $f^{(i-1)}$ that are covered by a seed of $f^{(i-1)}$.
The value of $f^{(i)}$ on those assignments is $\ast$.
By the seed lemma for DNF, every positive example
of $f$ is covered by a seed of some $f^{(i)}$ in this sequence.

For each $f^{(i)}$, 
the algorithm keeps a set of candidate seeds $T$ from $Q$.
With each such $T$ 
the algorithm keeps a term $T'$ (which includes the literals in $T$); it stores
the $(T,T')$ pairs in a set $H_i$.

The algorithm constructs a hypothesis DNF formula
made up of the terms $T'$ from the pairs $(T,T')$ in the $H_i$.
Intuitively,
the goal is to have each $H_i$ contain only pairs $(T,T')$ for actual seeds
$T$ of $f^{(i)}$, and for $T'$ to be the conjunction of $T$
and a monomial consistent with $f^{(i)}_T$.
Counterexamples
are used to modify the $H_i$ to get closer to this goal.

We present the details in the pseudocode in
Algorithm~\vref{eq}.   Note that the $T'$ are initialized to contain all literals,
and thus have no satisfying assignments.
The condition $T' \not\equiv 0$
means that $T'$ does not contain a variable and its negation.

\begin{algorithm}
\caption{EQ Algorithm}
\label{eq}

\begin{algorithmic}
	\STATE Initialize $h = 0$.  Ask an equivalence query with $h$.  If answer is yes \textbf{return $h$}, else let $e$ be the counterexample  received.
        \STATE for all $1 \leq j \leq |Q|$, $H_j = \{(T,T') \mid T \in Q, T' = \bigwedge_{l \in X \cup \bar{X}} l \}$ 
	\WHILE {True}
		\IF [$e$ is a positive counterexample] {$e$ does not satisfy $h$} 
			\FOR {$j=1$ \TO $|Q|$}
		           \IF {$e$ satisfies $T$ for some $(T,T') \in H_j$}
					\FORALL {$T$ such that $(T,T') \in H_j$ and $e$ satisfies $T$}
						\STATE remove from $T'$ all literals falsified by $e$
					\ENDFOR
					\STATE \textbf{break} out of \textbf{for} $j=1$ \TO $|Q|$ loop
                                        
			  \ENDIF
			\ENDFOR

		\ELSE [$e$ is a negative counterexample] 
	                       \FOR {$j=1$ \TO $|Q|$}
				\STATE Remove from $H_j$ all $(T,T')$ such that $T'$ is satisfied by $e$
			\ENDFOR
		\ENDIF 
                \STATE $H^* = \{ T':$ \text{for some} $j$, $(T,T') \in H_j$ \text{and} $T' \not\equiv 0$ \}
		\STATE $h = \bigvee_{T' \in H^*} T'$
                     \STATE Ask an equivalence query with hypothesis $h$.  If answer is yes, \textbf{ return} $h$, else let $e$ be the counterexample received.
	\ENDWHILE
	
\end{algorithmic}
\end{algorithm}

We now prove correctness.
It is easy to see that each hypothesis $h$ is consistent with
all positive counterexamples received so far.
For term $T$,
let $A_{T,i} = \{ e \in \{0,1\}^n| T(e) = 1$ and $f^{(i)}(e) = 1\}$,
and let
$M_{T,i} = \{ l \in X_n \bigcup \bar{X_n} | l$ is satified by all $e \in A_{T,i}\}$.
We prove that the following invariant holds:  For each $H_i$,
if $T$ is a seed of $f^{(i)}$, then $H_i$ contains a pair $(T,T')$
where $T'$ contains all literals in $M_{T,i}$ and $T$.
The invariant holds initially.
Assume it holds 
before processing of a counterexample $e$.
If $e$ is a positive counterexample, then each
resulting update modifies a
$T'$, where $(T,T') \in H_j$ for some $j$.
and $e$ satisfies $T$.
Suppose $T$ is a seed of $f^{(j)}$.
Let $i$ be the minimum value such that
$e$ is covered by a seed of
$f^{(i)}$.
By the invariant $j \leq i$ and $e$ is a positive example of $f^{(j)}$.
Hence $e \in A_{T,j}$ and satisfies
all literals in $M_{T,i}$, so the invariant holds after the update.

Now suppose $e$ is a negative counterexample.
If $e$ satisfies $T$ such
that
$(T,T') \in H_j$, and $T$ is a seed of $f^{(j)}$,
then
$f^{(j)}_T$ is consistent with a monomial, so every negative
example of $f$ must falsify $T$ or some literal in $M_{T,j}$.
Therefore, by the invariant, $e$ falsifies $T'$.  
Thus in processing $e$, a pair $(T,T')$ is removed from $H_j$ only if
$T$ is not a seed of $f^{(j)}$, so again the invariant is maintained.

Since each negative counterexample eliminates
a pair $(T,T')$ from some $H_j$, the number of negative counterexamples is
$2^{O(\sqrt{n \log s} \log n)}$.
Since each positive counterexample eliminates
at least one literal from $T'$, in some $(T,T')$, and $h$ is
always satisfied by the positive counterexamples,
the number of positive counterexamples is
$2^{O(\sqrt{n \log s} \log n)}$.
Thus the algorithm will output a correct hypothesis
in time  
$2^{O(\sqrt{n \log s} \log n)}$.  

We have proved the following theorem.

\begin{theorem}
There is an algorithm that EQ-learns DNF properly in time 
$2^{O(\sqrt{n \log s} \log n)}$.
\end{theorem}

Our algorithm can be viewed as an MEQ algorithm that does not make
membership queries.  The results of 
Hellerstein et al.~\cite{HellersteinRaghavan:05} relating certificates and query complexity
imply the following corollary.  We also present a direct proof, based
on the seed lemma for DNF, in Appendix~\ref{certproof}.

\begin{corollary}
\label{certcor}
There exists 
a function $r(s,n) = 2^{O(\sqrt{n\log s}\log n)}$ such that
for all $s,n > 0$, 
for all Boolean functions $f(x_1, \ldots, x_n)$, if $\ds(f) > r(s,n)$, then
$f$ has a certificate of size at most $r(s,n)$
certifying that $ds(f) > s$. 
\end{corollary}

\section{A tradeoff between number of queries and size of queries for properly learning DNF} \label{sec:tradeoff}

In this section we give a careful quantitative sharpening of Angluin's
approximate fingerprint proof, which showed that DNF cannot be properly
EQ-learned with polynomial query complexity~\cite{Angluin:90}.
We thereby prove 
a tradeoff between the number of queries
and the size of queries that a proper EQ algorithm must use.
Suppose that $A$ is any proper EQ algorithm for learning DNF .
We show that if $A$ does not use
hypotheses with many terms, then $A$ must make many queries.
Our result is the following (no effort has
been made to optimize constants):

\begin{theorem} \label{thm:tradeoff}Let $17 \leq k \leq
\sqrt{n/(2 \log n)}.$
Let $A$ be any EQ algorithm which learns the class of all
poly$(n)$-size DNF formulas using queries which are
DNF formulas with at most $2^{n/k}$ terms.
Then $A$ must make at least $n^{k}$
queries in the worst case.
\end{theorem}

Taking $k=\Theta(\sqrt{n/\log n})$
in Theorem~\ref{thm:tradeoff}, we see that any
algorithm that learns $\poly(n)$-term DNF using $2^{\sqrt{n \log n}}$-term 
DNF hypotheses must make at least $2^{\Omega(\sqrt{n \log n})}$ queries.  

\medskip

We use the following lemma, which is a quantitative sharpening of Lemma~5 of \cite{Angluin:90}.  The proof is in Appendix~\ref{prooflemma}.

\begin{lemma} \label{lem:l1}
Let $f$ be any $T$-term DNF formula over $n$ variables where $T \geq 1.$
For any $r \geq 1$,
either there is a positive assignment $y \in \{0,1\}^n$ (i.e. $f(y) = 1$) such that $|y| \leq
r \sqrt{n}$, or there is a negative assignment $z \in \{0,1\}^n$ (i.e. $f(z)=0$) such that
$n > |z| > n - (\sqrt{n} \ln T)/r - 1.$
\end{lemma}


\begin{proofof}{Theorem~\ref{thm:tradeoff}}
As in \cite{Angluin:90} we define $M(n,t,s)$ to be the class of all monotone
DNF formulas over variables $x_1,\dots,x_n$ with exactly $t$ distinct terms,
 each containing exactly $s$ distinct variables.
Let $M$ denote $\binom{\binom{n}{s}}{t}$, the number of formulas in $M(n,t,s).$

For the rest of the proof we fix $t = n^{17}$ and $s = 2k \log n.$ We will
show that for these settings of $s$ and $t$ the following holds:
given any DNF formula $f$ with at most $2^{n/k}$ terms,
there is some assignment $a^f \in \{0,1\}^n$ such that
at most $ M / {n^k}$ of the $M$ DNFs in $M(n,t,s)$
agree with $f$ on $a^f$. This implies that any EQ algorithm
using hypotheses that are DNF formulas with
at most $2^{n/k}$ terms must have query complexity at
least $n^k$ in the worst case  (By answering
each equivalence query $f$ with the counterexample $a^f$ as described above,
an adversary can cause each equivalence query to eliminate at most $M / n^k$
of the $M$ target functions in $M(n,s,t).$
Thus after $n^k - 1$ queries there must be at least $M/ n^k >1$
possible target functions in $M(n,t,s)$ that are still consistent with all
queries and responses so far, so the algorithm cannot be done.)

Recall that $17 \leq k \leq \sqrt{n/(2 \log n)}.$
Let $f$ be any DNF with at most $2^{n/k}$ terms.
Applying Lemma \ref{lem:l1} with $r=\sqrt{n}/2$, we get that either
there is a positive assignment $y$ for $f$ with $|y| \leq r \sqrt{n}
=n/2,$
or there is a negative assignment $z$ with
$n > |z| \geq n - (\sqrt{n} \ln (2^{n/k}))/r -1 =
n - {\frac {(2 \ln 2) n} k} - 1 \geq n - {\frac {3n} k}$.
Let $\phi$ be a DNF formula randomly and uniformly selected from $M(n,t,s).$
All probabilities below refer to this draw of $\phi$ from $M(n,t,s).$

We first suppose that there is a positive assignment $y$ for $f$
with $|y| \leq n/2.$ In this case the probability (over the random
choice of $\phi$) that any fixed term of $\phi$
(an AND of $s$ randomly chosen variables)
is satisfied by $y$ is exactly
$
{\frac
{\binom{y}{s}}
{\binom{n}{s} } \leq
\frac
{\binom{n/2}{s}}
{\binom{n}{s}} \leq \frac {1} {2^s}}.
$
A union bound gives that
$\Pr_\phi[\phi(y)=1] \leq t/2^s.$
Thus in this case, 
at most a $t/2^{s}$ fraction of
formulas in $M(n,t,s)$ agree with $f$ on $y.$  Recalling that $t=n^{17}$,
$s=2k \log n$ and $k \geq 17$, we get that $t/2^s \leq 1/n^k$
as was to be shown.

Next we suppose that there is a negative assignment $z$ for $f$ such that
$n > |z| \geq n(1 - {\frac 3 k}).$  At this point we recall the following
fact from \cite{Angluin:90}:

\begin{fact} [Lemma~4 of \cite{Angluin:90}]
\label{fact:binom}
Let $\phi$ be a DNF formula chosen uniformly at random
from $M(n,t,s).$  Let $z$ be an assignment which is such that
$t \leq {\binom{n}{s}} - {\binom{|z|}{s}}.$\footnote{The statement
of Lemma~4 of \cite{Angluin:90} stipulates that $t \leq n$ but it is
easy to verify from the proof that $t \leq {\binom{n}{s}} - {\binom{|z|}{s}}$ 
is all that is required.}
Then
$\Pr_{\phi}[\phi(z)=0] \leq
(1 - ((|z| - s)/n)^s)^t.$
\end{fact}

Since $t=n^{17}$, $|z|\leq n-1$, and $s = O(\sqrt{n \log n})$,
we indeed have that
$t \leq \binom{n}{s} - {\binom{|z|}{s}}$ as required by the above fact.
We thus have
\begin{eqnarray*}
\Pr_{\phi}[\phi(z)=0] &\leq&
\left(
1 - \left({\frac {n (1 - {\frac 3 k}) - s}{n}}\right)^s\right)^t =
\left(
1 - \left(1 - {\frac 3 k} - {\frac s n}\right)^s\right)^t.
\end{eqnarray*}
Recalling that $k \leq \sqrt{n/(2 \log n)}$ we have that $s/n = 2k \log n/n
\leq 1/k$, and thus
\[
\Pr_\phi[\phi(z) = 0] \leq
\left( 1 - \left(1 - {\frac 4 k} \right)^s\right)^t=
\left( 1 - \left(1 - {\frac 4 k} \right)^{2k \log n}\right)^{n^{17}}.
\]
Using the simple bound $(1 - {\frac 1 x})^x \geq 1/4$ for $x \geq 2$,
we get that $\left(1 - {\frac 4 k} \right)^{2k \log n} \geq 1/n^{16}.$
Thus we have
\[
\Pr_\phi[\phi(z) = 0] \leq
\left(1 - {\frac 1 {n^{16}}}\right)^{n^{17}}
\leq e^{-n} \ll {\frac 1 {n^k}}
\]
as was to be shown.  This concludes the proof of Theorem~\ref{thm:tradeoff}.
\end{proofof}

\section{Achieving this tradeoff between number of queries and query
size for properly learning DNF}

In this section we prove a theorem showing that the tradeoff between number of queries and
query size established in the previous section is essentially
tight.  Note that the algorithm $A$ described in the proof of the theorem
is not computationally efficient.

\begin{theorem} \label{thm:main}
Let $1 \leq k \leq {\frac {3n}{\log n}}$
and fix any constant $d>0.$
There is an algorithm $A$ which learns the class of all
$n^d$-term DNF formulas using at most $O(n^{k+d+1})$ DNF hypothesis
equivalence queries, each of which is an $2^{O(n/k)}$-term DNF.
\end{theorem}

Following \cite{BCG+:96}, the idea of the proof is to have each equivalence query be designed so as to eliminate at least a $\delta$ fraction of the remaining concepts in the class.  It is easy to see that $O(\log(|C|)\cdot \delta^{-1})$ such equivalence queries suffice to learn a concept class $C$ of size $|C|$.  Thus the main challenge is to show that there is always a DNF hypothesis having ``not too many'' terms which is guaranteed to eliminate many of the remaining concepts.  This is done by taking a majority vote over randomly chosen DNF hypotheses in the class, and then showing that this majority vote of DNFs can itself be expressed as a DNF with ``not too many'' terms.

\medskip
\begin{proofof}{{Theorem~\ref{thm:main}}}

At any point in the execution of the algorithm, let $CON$ denote
the set of all $n^d$-term DNF formulas that are consistent with
all counterexamples that have been received thus far (so
$CON$ is the ``version space'' of $n^d$-term DNF
formulas that could still be the target concept given what the algorithm
has seen so far).

A simple counting argument gives that
there are at most $3^{n^{d+1}}$ DNF formulas of length at most $n^d.$
We
describe an algorithm $A$ which makes only equivalence
queries which are DNF formulas with at most $n^{k}$ terms and,
with each equivalence query,
multiplies the size of $CON$
by a factor which is at most $\left(1 - {\frac 1 {n^k}}\right).$
After $O(n^{k+d+1})$ such queries the algorithm will have caused $CON$
to be of size at most 1, which means that it has succeeded in exactly
learning the target concept.

We first set the stage before describing the algorithm.
Fix any point in the algorithm's execution and let
$CON = \{f_1,\dots,f_N\}$ be the set of all consistent $n^d$-term DNF
as described above.  Given
an assignment $a \in \{0,1\}^n$ and a label $b \in \{0,1\}$, let
$N_{a,b}$ denote the number of functions $f_i$ in $CON$ such
that $f(a)=b$ (so for
any $a$ we have $N_{a,0} + N_{a,1} = N$), and let $N_{a,min}$ denote min$\{N_{a,0},N_{a,1}\}.$

Let $Z$ denote the set of those assignments $a \in \{0,1\}^n$ such that
$N_{a,min} < {\frac 1 {n^k}} \cdot N$, so an assignment is in $Z$ if the overwhelming
majority of functions in $CON$ (at least a $1 - {\frac 1 {n^k}}$ fraction) all give the same output
on the assignment.  We use the following claim, whose proof is in Appendix~\ref{proofclaim}.

\begin{claim} \label{claim:maj} There is a list of $t={\frac {3 n } {k \log n}}$
functions $f_{i_1},\dots,f_{i_t} \in CON$ which is such that the function
$\Maj(f_{i_1},\dots,f_{i_t})$ agrees with $\Maj(f_1,\dots,f_N)$ on all assignments $a \in Z.$
\end{claim}

By Claim~\ref{claim:maj} there must exist some function $h_{CON}=\Maj(f_{i_1},\dots,f_{i_t}),$ where
 each $f_{i_j}$ is an $n^d$-term DNF, which agrees with $\Maj(f_1,\dots,f_N)$ on all
assignments $a \in Z.$  The function $\Maj(v_1,\dots,v_t)$ over Boolean variables $v_1,\dots,v_t$
can be represented as a monotone $t$-DNF with at most $2^t$ terms.  If we substitute the $n^d$-term DNF $f_{i_j}$
for variable $v_j$, the result is a depth-4 formula with an OR gate at the top of fanin at most $2^t$,
AND gates at the next level each of fanin at most $t$, OR gates at the third level each of
fanin at most $n^d$, and AND gates at the bottom level.
By distributing to ``swap'' the second and third levels of the formula
from AND-of-OR to OR-of-AND and then collapsing the top two levels of adjacent OR gates and the bottom two levels of adjacent AND gates, we get that $h_{CON}$ is expressible as a DNF with
$
2^t \cdot n^{dt} = 2^{O(n /k)}
$
terms.

Now we can describe the algorithm $A$ in a very simple way:  at each point in its execution, when $CON$
is the set of all $n^d$-term DNF consistent with all examples received so far as described
above, the algorithm $A$ uses the hypothesis $h_{CON}$ described above as its
equivalence query.  To analyze the algorithm we consider two mutually exclusive possibilities for the
counterexample $a$ which is given in response to $h_{CON}$:

\medskip {\bf Case 1:}  $a \in Z.$  In this case, since $h(a)$ agrees with the
majority of the values $f_1(a),\dots,f_N(a)$, such a counterexample causes the
size of $CON$ to be multiplied by a number which is at most $1/2.$

\medskip {\bf Case 2:}  $a \notin Z.$  In this case we have $N_{a,0},N_{a,1} \geq {\frac 1 {n^k}}$ so
the counterexample $a$ must cause the size of $CON$
to be multiplied by a number which is at most $\left(1 - {\frac 1 {n^k}}\right).$
This proves Theorem~\ref{thm:main}.
\end{proofof}

\section{Membership queries provably help for learning $\log n$-term DNF}
The following is a sharpening of the arguments from Section~\ref{sec:tradeoff}
to apply to $\log(n)$-term DNF.

\begin{theorem} \label{thm:needmq}
Let $A$ be any algorithm which learns the class of all
$\log n$-term DNF formulas using only equivalence queries which are DNF formulas with
at most $n^{\log n}$ terms.  Then $A$ must make at least $n^{(\log n)/3}$
equivalence queries in the worst case.
\end{theorem}

\begin{sketchof}{Theorem~\ref{thm:needmq}}  As in the proof of Theorem~\ref{thm:tradeoff}
we consider $M(n,t,s)$, the class of all monotone DNF over $n$ variables with exactly
$t$ distinct terms each of length exactly $s$.  For this proof we fix
$s$ and $t$ both to be $\log n.$ We will show that given any DNF formula with
at most $n^{\log n}$ terms, there is an assignment such that at most a $1/n^{(\log n)/3}$ fraction
of the DNFs in $M(n,t,s)$ agree with $f$ on that assignment; this implies the theorem by the arguments
of Theorem~\ref{thm:tradeoff}.  Details are in Appendix~\ref{prooftheorem}.
\end{sketchof}

\bibliographystyle{plain}
\bibliography{seeds}

{\bf \noindent \LARGE{Appendices}}
    
\appendix

\section{Learning models}
\label{models}

In this appendix, we define the learning models used in this
paper.  We present the models here
only as they
apply to learning DNF formulas.
See e.g.~\cite{Angluin:92} for additional
information and more general definitions of the models.

In the \emph{PAC} learning model~\cite{Valiant:84},
a DNF learning algorithm is given as input parameters $\epsilon$ and $\delta$.  It is also given access to an oracle
$EX(c,\cal{D})$, for a  target DNF formula $c$ defined
on $X_n$ and a probability distribution
$\cal{D}$ over $\{0,1\}^n$.
On request, the oracle produces a {\em labeled example} $(x, c(x))$,
where $x$ is randomly generated with respect to $D$.
An algorithm $A$ 
{\em PAC-learns} DNF if for any DNF formula $c$ 
on $X_n$, any distribution $D$ on $\{0,1\}^n$, 
and any $0 < \epsilon, \delta < 1$, the following holds: Given $\epsilon$ and $\delta$,
and access to oracle $EX(c,\cal{D})$, with
probability at least $1 - \delta$, $A$ outputs
a hypothesis $h$ such
that Pr$_{x \in \cal{D}}[h(x) \neq c(x)] \leq \epsilon.$
Algorithm $A$ is a {\em proper} DNF-learning algorithm if
$h$ is a DNF formula.

In the EQ model~\cite{Angluin:88},
a DNF learning algorithm is 
given access to an oracle that answers {\em equivalence
queries} for a target DNF formula $c$ defined on $X_n$.
An equivalence query asks
``Is $h$ equivalent to target $c$?'', where $h$ is a hypothesis. 
If $h$ represents the same function as $c$, 
the answer is ``yes,'' otherwise, the answer is a {\em counterexample} 
$x \in \{0,1\}^n$ such that $h(x) \neq c(x)$. 
If $c(x) = 1$, $x$ is a
\emph{positive counterexample} else it is a
\emph{negative counterexample}.
Algorithm $A$ {\em EQ-learns} DNF
if, for $n > 0$
and any DNF formula $c$ defined on $X_n$, 
the following holds:
if $A$ is given access
to an oracle answering equivalence queries for $c$, then $A$ outputs
a hypothesis $h$ representing exactly the same function as $c$.
Algorithm $A$ EQ-learns DNF {\em properly} if
all hypotheses used (in equivalence queries, and in the output)
are DNF formulas.

A PAC or EQ learning algorithm
{\em learns $k$-term DNF} if it satisfies the relevant
requirements above when
the target is restricted to be a $k$-term DNF formula.

In variants of the PAC and EQ models,
the learning algorithm can ask {\em membership queries}
which ask
``What is $c(x)$?'' for target $c$ and assignment $x$.
The answer is the value of $c(x)$.

A PAC algorithm for learning DNF is said to run in time $t=t(n,s,\epsilon,\delta)$
if it takes at most
$t$ time steps, 
and its output hypothesis can be evaluated on
on any point in its domain in time $t$, when
the target is over $\{0,1\}^n$ and has size $s$.
The time complexity for EQ algorithms is defined analogously for $t = t(n,s)$.

The {\em query complexity} of an EQ learning algorithm is the sum of the sizes
of all hypotheses used.


\section{Pseudocode for PAC algorithm}
\label{pacalgorithm}

Pseudocode for the PAC algorithm of Section \ref{PACsection}:

\begin{algorithm}[h!btp]
\caption{PAC algorithm}
\label{pac}
\begin{algorithmic}
	\STATE  $X = \{x_1, \ldots, x_n\}, \bar{X} = \{\bar{x}_1, \ldots, \bar{x}_n\}$
	\STATE $Q = \{t \subset X \cup \bar{X} \mid |t| \leq 2 \sqrt{n \ln s}\}$ \COMMENT set of potential seeds
	\STATE $h = 0$
	\WHILE {$Q \neq \emptyset$ AND $S^+ \neq \emptyset$}
                       \FORALL { $t \in Q$}
                                \STATE T = $\bigwedge_{l \in t} {l}$
			\IF [test $T$ to see if it is a seed of $f^S$] {$T$ covers at least one $e \in S^+$}
				\STATE $S_T = \{e \mid e \in S^+$ AND $T$ covers $e$ $\}$
				\STATE $T' = \bigwedge_{l \in B} l$ where $B = \{ {l  \in X \cup \bar{X} \mid x}$\text{   is satisified by all } $e \in S_T\}$.
				\IF {$\{e \mid e \in S^-$ AND $e$ satisfies $T'\} = \emptyset$}  
							\STATE $S^+ = S^+ \setminus S_T$
							\STATE $h = h \vee T'$
                                                                          \STATE Remove $t$ from $Q$
				\ENDIF
			\ENDIF
                             \ENDFOR
		\ENDWHILE
		\IF {$S^+ \neq \emptyset$ }			\RETURN \textbf{fail}
	\ELSE			\RETURN $h$
	\ENDIF
	
	\end{algorithmic}
\end{algorithm}
\section{Subexponential certificates for functions of more than subexponential DNF size}
\label{certproof}

We present a direct proof of Corollary~\ref{certcor}, based on the seed lemma for DNF.

\begin{proof}
Let $s,n > 0$.
Let $q(s,n) = 2\sqrt{n \log s}$.  
Let $f$ be a function on $n$ variables such that $\ds(f) >n^{q(s,n)}$.
We first claim that there exists a 
partial function $f'$, created by removing a
subset of the positive examples from $f$ and setting
them to be undefined,
 that does not have a seed of size at most $q(s,n)$.  
Suppose for contradiction that all
such partial functions $f'$ have such a seed.
Let $S$ be the sample consisting of all $2^n$ labeled examples $(x,f(x))$ of $f$.
We can apply the seed covering method of
Section~\ref{PACsection} to produce a DNF consistent with $f$, using a seed of size
$q(s,n)$ at every stage.  Since no seed will be used more than once,
the output DNF is bounded by the number of terms of size at most $q(s,n)$,
which is less than $n^{q(s,n)}$.
This contradicts that $\ds(f) > n^{q(s,n)}$.
Thus the claim holds, and $f'$ exists.

Since $f'$ does not have a seed of size at most $q(s,n)$,
each term $T$ of size at most $q(s,n)$ either does not cover any
positive examples of $f'$, or 
the projection $f'_T$ is not consistent
with a monomial.
Every function (or partial function) that is not consistent with a monomial has a certificate of size 3 certifying that it has that property, consisting
of two positive examples of
the function, and a negative example that is {\em between} them (cf.~\cite{EHP:97}).  For assignments $r,x,y \in \{0,1\}^n$, we
say that $r$ is between $x$ and $y$ if
$\forall i$, $p_i = r_i$ or $q_i = r_i$. It follows that if $f'_T$ is not consistent with a monomial, then
$f'$ has a certificate $c(T)$ of size 3 proving that fact, 
consisting of two positive examples of $f'$ that satisfy $T$,
and one negative example of $f'$ satisfying $T$ that is between them.

Let ${\cal T} =  \{ T |$ term $T$ is such that $|T| \leq q(s,n)$
and $f'_T$ is not consistent with a monomial$\}$.  
Let $A = \bigcup_{T \in {\cal T}} c(T)$.
Clearly $|A| < 3n^{q(s,n)}$.
We claim that $A$
is a certificate that $\ds(f) > s$.
Suppose not.  Then there exists a function $g$ that is consistent
with $f$ on the assignments in $A$, such that $\ds(g) \leq s$.
Consider the partial function $h$ which is defined
only on the assignments in $A$, and is consistent with 
$g$ (and $f$) on those assignments.
The partial function $h$ does not have a seed of size at most $q(s,n)$,
because for all terms $T$ of size at most $q(s,n)$, either
$T$ does not cover a positive assignment of $h$, 
or $A$ contains a certificate that $h_T$ is not consistent with a monomial.
Since $\ds(g) \leq s$, and every DNF that is consistent with $g$
is also consistent with $h$, $\ds(h) \leq s$ also.
Thus by the seed lemma for DNF, $h$ has a seed of size at most $q(s,n)$.
Contradiction.

\end{proof}

\section{Proofs}
\subsection{Proof of Lemma~\ref{lem:l1}}
\label{prooflemma}

\begin{proofof}{{Lemma~\ref{lem:l1}}}
The proof uses the following claim, which is established by a simple greedy argument:

\begin{claim} [Lemma~6 of \cite{Angluin:90}] \label{claim:angluin}
Let $\phi$ be a DNF formula with $T\geq1$ terms such that each term contains at least $\alpha n$
distinct unnegated variables, where $0 < \alpha < 1.$  Then there is a
nonempty\footnote{We stress that $V$ is nonempty because this will
be useful for us later.} set $V$ of at most
$1 + \lfloor \log_b T \rfloor$ variables such that each term of $\phi$ contains a positive
occurrence of some variable in $V$, where $b = 1/(1 - \alpha).$
\end{claim}

Let $f$ be a $T$-term DNF formula.  Since by assumption we have $T \geq 1$, there is at least one term in $f$ and hence at least one positive assignment $y$ for $f$.  If $r \geq \sqrt{n}$ then clearly this
positive assignment $y$ has $|y| \leq
r \sqrt{n}$, so the lemma holds for $r \geq \sqrt{n}$.  Thus we may henceforth assume that
$r < \sqrt n.$

Let $\alpha = {\frac {r} {\sqrt{n}}}$ (note that $0 < \alpha < 1$ as required by Claim~\ref{claim:angluin}).
If there is some term of $f$ with fewer than $\alpha n = r \sqrt{n}$ distinct unnegated variables,
then we can obtain a positive assignment $y$ for $f$ with $|y| < r \sqrt{n}$ by
setting exactly those variables to 1 which are unnegated in this term and setting
all other variables to 0.  So we may suppose that every term of $f$ has at least $\alpha n$
distinct unnegated variables.  Claim~\ref{claim:angluin}  now implies that
there is a nonempty set $V$ of at most
$$
1 + \lfloor \log_{1/(1 - r/\sqrt{n})} T \rfloor \leq 1 + {\frac {\sqrt{n}} r} \ln T
$$
variables $V$ such that each term of $f$ contains a positive
occurrence of some variable in $V$.  The assignment $z$ which sets all and only the variables in $V$ to 0 is a negative assignment
with $n > |z| \geq n - (\sqrt{n} \ln T)/r - 1$ (note that
$n > |z|$ because $V$ is nonempty), and Lemma~\ref{lem:l1} is proved.
\end{proofof}

\subsection{Proof of Claim~\ref{claim:maj}}
\label{proofclaim}

\begin{proof}
Let functions $f_{i_1},\dots,f_{i_t}$ be drawn independently and uniformly
from $CON$.  (Note that $t \geq 1$ by the bound $k \leq {\frac {3n}{\log n}}.$)  We show that with nonzero probability the resulting list
of functions has the claimed property.

Fix any $a \in Z.$  The probability that
$\Maj(f_{i_1},\dots,f_{i_t})$ disagrees with $\Maj(f_1,\dots,f_N)$ on $a$ is easily seen to be at most
$$
\binom{t}{t/2} \left({\frac 1 {n^k}}\right)^{t/2} < {\frac {2^t} {n^{kt/2}}}.
$$
Recalling that $t={\frac {3n}{k \log n} }$, this is less than $1/2^n$ for all $1 \leq k \leq n.$
Since there are at most $2^n$ assignments $a$ in $Z$, a union bound over all $a \in Z$ gives that
with nonzero probability (over the random draw of $f_{i_1},\dots,f_{i_t}$) the function
$\Maj(f_{i_1},\dots,f_{i_t})$ agrees with $\Maj(f_1,\dots,f_N)$ on all assignments in $Z$
as claimed. \end{proof}

\subsection{Proof of Theorem~\ref{thm:needmq}}
\label{prooftheorem}

\begin{proofof}{Theorem~\ref{thm:needmq}}  
Let $M(n,t,s)$ be the class of all monotone DNF over $n$ variables with exactly
$t$ distinct terms each of length exactly $s$.  Fix
$s$ and $t$ both to be $\log n.$ We will show that given any DNF formula with
at most $n^{\log n}$ terms, there is an assignment such that at most a $1/n^{(\log n)/3}$ fraction
of the DNFs in $M(n,t,s)$ agree with $f$ on that assignment; this implies the theorem by the arguments
of Theorem~\ref{thm:tradeoff}.  

Let $f$ be any DNF formula with at most $T=n^{\log n}$ terms.
Applying Lemma~\ref{lem:l1} to $f$ with $r=1$, we may
conclude that either there is an assignment $y$ with $|y| \leq \sqrt{n}$
and $f(y)=1,$ or there is an assignment $z$ with $n > |z| \geq n - \sqrt{n} (\log n)^2$ and $f(z)=0.$

Let $\phi$ be a DNF formula randomly and uniformly selected from $M(n,t,s).$
All probabilities below refer to this draw of $\phi$ from $M(n,t,s).$

We first suppose that there is an assignment $y$ with $f(y) = 1$ and $|y| \leq \sqrt{n}.$
The probability that any fixed term of $\phi$
(an AND of $s$ randomly chosen variables)
is satisfied by $y$ is exactly
$$
{\frac
{{\binom{|y|}{s}}}
{{\binom{n}{s}}}  \leq
{\frac
{{\binom{\sqrt{n}}{s}}}
{\binom{n}{s}}}} < \left( {\frac 1 {\sqrt{n}}}\right)^s = {\frac 1 {n^{(\log n)/2}}}.
$$
A union bound gives that $\Pr_\phi[\phi(y)=1] \leq t \cdot {\frac 1 {n^{(\log n)/2}}} <
{\frac 1 {n^{(\log n)/3}}}.$  So
in this case $y$ is an
assignment such that at most a ${\frac 1 {n^{(\log n)/3}}}$
fraction of formulas in $M(n,t,s)$ agree with $\phi$ on $y.$

Next we suppose that there is an assignment $z$ with $f(z)=0$ and
and $n > |z| > n - \sqrt{n} (\log n)^2.$  Since $s=t=\log n$ and
and $|z| \leq n-1$, we have that $t \leq \binom{n}{s} - \binom{|z|}{s}$ as required
by Fact~\ref{fact:binom}.  Applying Fact~\ref{fact:binom}, we get that
\begin{eqnarray*}
\Pr_{\phi}[\phi(z)=0] &\leq&
\left(
1 - \left({\frac {n -  \sqrt{n} (\log n)^2 - \log n }{n}}\right)^{\log  n}\right)^{\log n} \\
&<&
\left(
1 - \left({\frac {n -  2\sqrt{n} (\log n)^2 }{n}}\right)^{\log  n}\right)^{\log n}\\
&=&
\left(
1 - \left(1 - {\frac {2(\log n)^2}{\sqrt{n}}}\right)^{\log  n}\right)^{\log n}\\
&\leq&
\left(
1 - \left(1 -  {\frac {2(\log n)^3 }{\sqrt{n}}}\right)\right)^{\log n}\\
&=&
\left({\frac {2(\log n)^3 }{\sqrt{n}}}\right)^{\log n}
< \left({\frac 1 {n^{1/3}}}\right)^{\log n} =  {\frac 1 {n^{(\log n)/3}}}.
\end{eqnarray*}
So in this case $z$ is an assignment such that at most a
$1/n^{(\log n)/3}$ fraction of formulas in $M(n,t,s)$ agree
with $\phi$ on $z.$
This concludes the proof of Theorem~\ref{thm:needmq}.
\end{proofof}

\end{document}